\documentclass[twoside,11pt]{article}

\usepackage{jmlr2e}
\usepackage{epsfig}
\usepackage{graphicx}
\usepackage{algorithm}
\usepackage{algorithmic,amssymb}
\usepackage{amsmath}\allowdisplaybreaks
\usepackage{natbib}
\usepackage{color}

\begin{document}

\title{  Zeroth-order Asynchronous  Doubly  Stochastic  Algorithm with  Variance Reduction}

\author{\name Bin Gu \email jsgubin@gmail.com      \\
   \name Zhouyuan Huo \email {zhouyuan.huo@mavs.uta.edu} \\
   \name Heng Huang \email {heng@uta.edu} \\
\addr Department of Computer Science and Engineering\\
       University of Texas at Arlington\\
      }

\editor{}

\maketitle

\begin{abstract}
 Zeroth-order (derivative-free) optimization attracts a lot of attention in machine learning,  because   explicit gradient
calculations  may be computationally expensive or infeasible. To handle large scale problems both in volume and dimension, recently asynchronous doubly stochastic zeroth-order   algorithms were proposed. The  convergence rate of existing asynchronous doubly stochastic  zeroth order algorithms  is $O(\frac{1}{\sqrt{T}})$ (also for the sequential stochastic  zeroth-order optimization algorithms).  In this paper, we focus on the finite sums of smooth but not necessarily convex functions, and propose an asynchronous doubly stochastic   zeroth-order optimization algorithm using the accelerated technology of variance reduction  (AsyDSZOVR). Rigorous theoretical analysis show that  the convergence rate can be improved from $O(\frac{1}{\sqrt{T}})$  the  best result of existing algorithms  to  $O(\frac{1}{T})$. Also our theoretical results is an improvement to the ones of the sequential stochastic  zeroth-order optimization algorithms.
\end{abstract}


\section{Introduction}
 Zeroth-order (derivative-free) optimization attracts a lot of attention in machine learning,  because   explicit gradient
calculations  may be computationally expensive or infeasible. As we know, for a lot of machine learning optimization problems, such as graphical model inference \citep{wainwright2008graphical}, structured-prediction \citep{taskar2005learning}, and so on, it is difficult to give the explicit derivatives for the objective functions. For some black box learning model, such as black box neural networks \citep{lian2016comprehensive}, it is infeasible to  give the explicit derivatives. Also, for  bandit problems \citep{bubeck2012regret}, such as advertisement selection for search engines, it is infeasible to  give the explicit derivatives  of the objective functions because  only  observations of function values are available. Since zeroth-order methods estimate gradient  based on only  two point observations, it  is the best and only  choice of the optimization for  above scenarios.

Because the era of big data has
arrived, asynchronous parallel algorithms  for stochastic optimization have received huge successes in  theory and practice recently. Most of these  asynchronous parallel stochastic algorithms  are built on the first-order derivative or second-order information (e.g.  (approximate) Hessian matrix)  of the objective function. For example, Hogwild! \citep{recht2011hogwild} (the first lock-free asynchronous parallel stochastic gradient descent (SGD) algorithm) uses the first-order derivative to update the solution for smooth convex  functions. The other variants  of asynchronous parallel SGD algorithm \citep{mania2015perturbed,lian2015asynchronous,huo2016asynchronous,zhao2016fast} also use the first-order derivative to update the solution for smooth convex or nonconvex functions. For a composite of a smooth (possibly non-convex) function and a non-smooth convex function, the first-order derivative is embedded in the proximal operator \citep{Razaviyayn2014ParallelSC,liu2015asynchronous,you2016asynchronous}.  Also,   second-order information (e.g. (approximate) Hessian matrix) \citep{byrd2016stochastic} can be used to accelerate the optimization.

As the reasons  mentioned previously, designing asynchronous  stochastic  zeroth order algorithms is important and  urgent. As far as we know, the only work of  asynchronous  stochastic  zeroth order algorithm (AsySZO) is \citep{lian2016comprehensive}. They prove the convergence rate  $O(\frac{1}{T} + \frac{1}{\sqrt{T}})$. To the best of our knowledge, the convergence rates of existing  sequential  stochastic  zeroth order algorithms \citep{nesterov2011random,jamieson2012query,duchi2012randomized,agarwal2011stochastic} are $O(\frac{1}{T} + \frac{1}{\sqrt{T}})$ or $O \left (\frac{1}{\sqrt{T}} \right )$. Basically, the convergence rates of these algorithms can be viewed as $O \left (\frac{1}{\sqrt{T}} \right )$ because the term $\frac{1}{\sqrt{T}}$ dominates $\frac{1}{T} + \frac{1}{\sqrt{T}}$. Motivated by improving the convergence rate of SGD from $O \left (\frac{1}{\sqrt{T}}\right )$ to $O \left  (\frac{1}{T} \right )$, it is highly desirable  to design an accelerated asynchronous  stochastic  zeroth order algorithm with the convergence rate $O(\frac{1}{T})$.

 In this paper, we focus on  the finite sums of smooth but not necessarily convex functions   as follows.
   \begin{eqnarray}\label{formulation1}
\min_{x\in \mathbb{R}^N} f(x) =  \frac{1}{l}\sum_{i=1}^{l} f_i(x)
\end{eqnarray}
where  $f_i:\mathbb{R}^N \mapsto \mathbb{R}$ is a smooth, possibly
non-convex function function. The formulation (\ref{formulation1}) covers  an extensive number of  machine learning problems, for example, logistic regression \citep{freedman2009statistical}, ridge regression \citep{shen2013novel}, least squares SVM \citep{suykens1999least} and so on.

In this paper, we propose an asynchronous doubly stochastic   zeroth-order optimization algorithm using the accelerated technology of variance reduction  (AsyDSZOVR).  Our AsyDSZOVR randomly select a set of samples  and a set of  features simultaneously to handle large scale problems both in volume and dimension.
 Rigorous theoretical analysis show that  the convergence rate can be improved from $O(\frac{1}{\sqrt{T}})$  the  best result of existing algorithms  to  $O(\frac{1}{T})$. Also our theoretical results is an improvement to the ones of the sequential stochastic  zeroth-order optimization algorithms.

We organize the rest of the paper as follows.  In section \ref{section_algorithm}, we propose our AsySBCDVR algorithm. In Section \ref{convergence_analysis}, we prove the  convergence rate for AsySBCDVR.  Finally, we give some concluding
remarks in Section \ref{conclusion}.

 \section{Algorithms}\label{section_algorithm}
 In this section, we propose our AsyDSZOVR. In this paper, we focus on  the parallel environment with shared memory, such as  multi-core processors and GPU-accelerators, without any lock. Because the parallel computing pattern in the parallel environment with distributed memory can be  equivalent to the one in the  parallel environment with shared memory having reading and writing locks, our AsyDSZOVR can also work in the parallel environment with distributed memory.

 The basic parallel computing pattern includes  three steps, i.e., read, compute, update.
  Specifically, if the parallel computing is asynchronous,  all cores  repeat the
three steps independently and concurrently without any lock. We give a more detailed descriptions of the three steps as following.
 \begin{enumerate}
 \item \textbf{Read:} Read the vector $x$ from the shared memory to the local memory
without reading lock.
 \item \textbf{Compute:} Randomly  choose a component function $f_i$ or a mini-batch $\mathcal{B}$ of the component functions, and a  set of coordinates $J$, and locally compute an unbiased (approximate) gradient.
 \item \textbf{Update:} Update the set of coordinates $J$ of the vector $x$ in the shared memory, based on the  unbiased (approximate) gradient  without writing lock.
 \end{enumerate}
 To highlight the differences of AsySZO and our proposed AsyDSZOVR, we first give brief review of AsySZO, and present our AsyDSZOVR based on the above framework of parallel computing.  We also summarize the differences of of AsySZO and  AsyDSZOVR in Table \ref{table:methods}.
  \begin{table*}[ht]
 \center
 \caption{Comparisons  of AsySZO and AsyDSZOVR.} \setlength{\tabcolsep}{1mm}
\begin{tabular}{c|c|c|c|c|c}
\hline
\textbf{Algorithm}  & \textbf{Accelerated}  &  \textbf{Step size} &    \textbf{Mini-batch} & \textbf{$\widehat{x}_t - {x}_t$ or $\widehat{x}^{s+1}_t - {x}^{s+1}_t$ } & \textbf{ Rate}   \\ \hline
AsySZO  & No & Dynamic vanishing & No &  $\gamma \sum_{t' \in K(t)}    G_{J(t')}(\widehat{x}_t;f_i)$  & $O \left (\frac{1}{\sqrt{T}}\right )$  \\
AsyDSZOVR  & Yes & Constant & Yes &  $ \gamma \sum_{t' \in K(t)}B_{t'}^{s+1}   \widehat{v}^{s+1}_{J(t')}$ & $O \left  (\frac{1}{T} \right )$  \\
 \hline
\end{tabular}
\label{table:methods}
\end{table*}
 \subsection{Brief Review of AsySZO}
Actually, the existing  asynchronous  stochastic  zeroth order algorithm (i.e., AsySZO) proposed by \citep{lian2016comprehensive} strictly follows the three steps. Specifically, the unbiased (approximate) gradient in the `\textbf{Compute}' step is computed based on a randomly choosed component function $f_i$ as  \begin{eqnarray}
G_{ J}(x;f_i) = \sum_{j \in J} \frac{N}{2 Y \mu_j } \left ( f_i (x+ \mu_j e_j) - f_i (x- \mu_j e_j) \right )e_j
\end{eqnarray}
where $\mu_j$ is the approximate parameter for the $j$-th coordinate, and $e_j$ is the  zero vector in $\mathbb{R}^N$ except that the coordinates indexed by $j$ equal to $1$. Thus, the updating rule in the `\textbf{Update}' step is $(x_{t+1}^{s+1})_{J} \leftarrow  \left ( (x_{t}^{s+1}) - \gamma G_{J}(\widehat{x}^{s+1}_t;f_i) \right )_{J} $, where $\gamma$ is the step size. The pseudocode  of AsySZO can be found in Algorithm \ref{algorithm1}.

Because AsySZO does not use the reading  and writing locks,  the vector $\widehat{x}^{s+1}_t$ read into the local  memory may be inconsistent to the vector ${x}^{s+1}_t$ in the shared memory, which means that some components of $\widehat{x}^{s+1}_t$ are same with the ones in ${x}^{s+1}_t$, but others are different to the ones in  ${x}^{s+1}_t$. In \citep{lian2016comprehensive}, they present ${x}^{s+1}_t$ as following.
  \begin{eqnarray} \label{pi_test1.1}
{x}_t = \widehat{x}_t - \gamma \sum_{t' \in K(t)}    G_{J(t')}(\widehat{x}_t;f_i)
\end{eqnarray}
where $K(t)$ is a  set  of iterations. As mentioned in \citep{mania2015perturbed,zhao2016fast}, this  representation could not formulate the conflicts of two writing operations. For AsyDSZOVR, we will give a more reasonable representation of ${x}^{s+1}_t$.

\begin{algorithm}
\renewcommand{\algorithmicrequire}{\textbf{Input:}}
\renewcommand{\algorithmicensure}{\textbf{Output:}}
\caption{Asynchronous    Stochastic  Zeroth-order Optimization  (AsySZO)}
\begin{algorithmic}[1]
\REQUIRE $\gamma$,  $S$, and $m$.
\ENSURE $x^{S}$.
 \STATE  Initialize  $x^0 \in \mathbb{R}^d$, $p$ threads.
 \STATE \textit{For each thread}, do:
\FOR{$t=0,1,2,m-1$}
\STATE Randomly select a component function  $f_i$ from $\{1, ...,l\}$ with equal probability.
\STATE Randomly choose a set of coordinates  $J(t)$ from $\{1, ...,n\}$ with equal probability.

 \STATE   $(x_{t+1}^{s+1})_{J(t)} \leftarrow  \left ( (x_{t}^{s+1}) - \gamma G_{J(t)}(\widehat{x}^{s+1}_t;f_i) \right )_{J(t)} $. 

\STATE $(x_{t+1}^{s+1})_{\setminus J(t)} \leftarrow (x_{t}^{s+1})_{\setminus J(t)} $.
\ENDFOR
\end{algorithmic}
\label{algorithm1}
\end{algorithm}
\subsection{AsyDSZOVR}
Although $G_{ J}(x;f_i)$ is an unbiased estimate of $G_{ J}(x;f)$, it would have a large variance because it is computed based on one sample. Similar with \citep{huo2016asynchronous,zhao2016fast},  we use the variance reduction to accelerate AsySZO. Thus, AsyDSZOVR has two-layer loops. The outer layer is to parallelly compute the full  approximate gradient
 $G_{ J}(x^s;f) = \frac{1}{l} \sum_{i=1}^l G_{ J}(x^s;f_i)$, where the superscript $s$ denotes the $s$-th outer loop. The inner layer is to parallelly and repeatedly  update the vector $x$ in the shared memory, which also strictly follows the three steps as mentioned previously.
 Specifically,  all cores repeat the
following  steps independently and concurrently without any lock:
 \begin{enumerate}
 \item \textbf{Read:}  Read the vector $x$ from the shared memory to the local memory
without reading lock. We use $\widehat{x}^{s+1}_t$ to denote its value, where the subscript $t$ denotes the $t$-th inner loop.
 \item \textbf{Compute:} Randomly  choose a mini-batch $\mathcal{B}(t)$ of the component functions, and a  set of coordinates $J(t)$  from $\{1, ...,N\}$, and locally compute  $ \widehat{v}^{s+1}_{J(t)} = \frac{1}{|\mathcal{B}(t)|}\sum_{i\in \mathcal{B}(t)}  G_{J(t)}(\widehat{x}^{s+1}_t;f_i)- \frac{1}{|\mathcal{B}(t)|}\sum_{i\in \mathcal{B}(t)} G_{J(t)}(\widetilde{x}^s;f_i)  + G_{J(t)}(\widetilde{x}^s;f)$.
 \item \textbf{Update:} Update the set of coordinates $J(t)$ of the vector $x$ in the shared memory as $(x_{t+1}^{s+1})_{J(t)} \leftarrow  \left ( (x_{t}^{s+1}) - \gamma \widehat{v}^{s+1}_{J(t)} \right )_{J(t)} $   without writing lock.
\end{enumerate}
The detailed description of AsyDSZOVR is  presented in  Algorithm \ref{algorithm2}.  Note that $\widehat{v}^{s+1}_{J(t)} $ computed  locally is an approximation of $G_{ J}(\widehat{x}^{s+1}_{t};f)$, and the expectation  of $\widehat{v}^{s+1}_{J(t)} $ on $\mathcal{B}(t)$ is equal to $G_{ J}(\widehat{x}^{s+1}_{t};f)$ as shown  below.
\begin{eqnarray} \nonumber
\mathbb{E}_{\mathcal{B}(t)}\widehat{v}^{s+1}_{J(t)} & =& \mathbb{E}_{\mathcal{B}(t)}\left ( \frac{1}{|\mathcal{B}(t)|}\sum_{i\in \mathcal{B}(t)}  G_{J(t)}(\widehat{x}^{s+1}_t;f_i)- \frac{1}{|\mathcal{B}(t)|}\sum_{i\in \mathcal{B}(t)} G_{J(t)}(\widetilde{x}^s;f_i)  + G_{J(t)}(\widetilde{x}^s;f) \right )
\\ &=&  G_{J(t)}(\widehat{x}^{s+1}_t;f) - G_{J(t)}(\widetilde{x}^s;f) + G_{J(t)}(\widetilde{x}^s;f)   \nonumber
\\ &=& G_{J(t)}(\widehat{x}^{s+1}_t;f)
\end{eqnarray}
$\widehat{v}^{s+1}_{J(t)} $ is called a stochastic approximation  of $ G_{J(t)}(\widehat{x}^{s+1}_t;f)$. More importantly, we give an upper bound for $\sum_{t=0}^{m-1} \mathbb{E}  \left \|  \widehat{v}^{s+1}_{t} \right \|^2 $ (Lemma \ref{lemma1}). The lemma shows that $\widehat{v}^{s+1}_{J(t)} $ would  vanish after a large number of iterations. Thus, the step size $\gamma$ can be set as a fixed constant, which is different to the one used in AsySZO.

As mentioned in before, $\widehat{x}_t - {x}_t$ used in \cite{lian2016comprehensive}  could not formulate the conflicts of two writing operations. In this paper, we use the following formulation to present $\widehat{x}^{s+1}_t - {x}^{s+1}_t$.
  \begin{eqnarray} \label{pi_test1}
{x}^{s+1}_t = \widehat{x}^{s+1}_t - \gamma \sum_{t' \in K(t)}B_{t'}^{s+1}   \widehat{v}^{s+1}_{J(t')}
\end{eqnarray}
where $K(t)$ is a set  of  inner iterations, $t' \leq t-1$,  $B_{t'}^{s+1}$ is   a diagonal matrix with
diagonal entries either $1$ or $0$ ($0$ denotes that the corresponding coordinate is overwritten by other thread). It is reasonable to assume that  there exists an upper bound $\tau$ such that $\tau \geq t - \min\{t' | t' \in K(t)\}$ (i.e., Assumption \ref{ass0.1}).

\begin{assumption}[Bound of delay]\label{ass0.1}
There exists a  upper bound $\tau$ such that $\tau \geq t - \min\{t' | t' \in K(t)\}$ for all inner iterations $t$ in AsyDSZOVR.
 \end{assumption}

\begin{algorithm}
\renewcommand{\algorithmicrequire}{\textbf{Input:}}
\renewcommand{\algorithmicensure}{\textbf{Output:}}
\caption{Asynchronous  Doubly  Stochastic  Zeroth-order Optimization with  Variance Reduction  (AsyDSZOVR)}
\begin{algorithmic}[1]
\REQUIRE $\gamma$,  $S$, and $m$.
\ENSURE $x^{S}$.
 \STATE  Initialize  $x^0 \in \mathbb{R}^d$, $p$ threads.
\FOR{$s=0,1,2,S-1$}
\STATE $\widetilde{x}^s\leftarrow x^s$
\STATE \textit{All threads parallelly} compute the full fake gradient $G(\widetilde{x}^s;f) =  \sum_{i =1}^l \frac{1}{ l} G(\widetilde{x}^s;f_i)$
\STATE \textit{For each thread}, do:
\FOR{$t=0,1,2,m-1$}
\STATE Randomly sample a mini-batch $\mathcal{B}(t)$ from $\{1, ...,l\}$ with equal probability.
\STATE Randomly choose a set of coordinates  $J(t)$ from $\{1, ...,n\}$ with equal probability.
\STATE Compute $ \widehat{v}^{s+1}_{J(t)} = \frac{1}{|\mathcal{B}(t)|}\sum_{i\in \mathcal{B}(t)} G_{J(t)}(\widehat{x}^{s+1}_t;f_i)- \frac{1}{|\mathcal{B}(t)|}\sum_{i\in \mathcal{B}(t)} G_{J(t)}(\widetilde{x}^s;f_i)  + G_{J(t)}(\widetilde{x}^s;f)$.
 \STATE   $(x_{t+1}^{s+1})_{J(t)} \leftarrow  \left ( (x_{t}^{s+1}) - \gamma \widehat{v}^{s+1}_{J(t)} \right )_{J(t)} $. 

\STATE $(x_{t+1}^{s+1})_{\setminus J(t)} \leftarrow (x_{t}^{s+1})_{\setminus J(t)} $.
\ENDFOR
\STATE ${x}^{s+1}\leftarrow x_{m}^{s+1}$
\ENDFOR
\end{algorithmic}
\label{algorithm2}
\end{algorithm}

\section{Convergence Analysis} \label{convergence_analysis}
In this section, we  prove the convergence  rate of   AsyDSZOVR (Theorem \ref{theorem2} and Corollary \ref{corollary1}). Specifically, we improve the convergence rate of asynchronous stochastic  zeroth-order optimization from $O(\frac{1}{\sqrt{T}})$ to  $O(\frac{1}{T})$. If AsyDSZOVR only uses one thread,  AsyDSZOVR degenerates to the sequential doubly stochastic zeroth-order optimization algorithm with variance
reduction (DSZOVR). Our theoretical analysis can work for this condition, and we have the convergence rate $\frac{1}{T}$ for DSZOVR (Corollary \ref{corollary2}). It is also an improvement to the convergence rates of the existing  sequential stochastic
zeroth-order optimization algorithms \citep{nesterov2011random,jamieson2012query,duchi2012randomized,agarwal2011stochastic}.

Before providing  the theoretical analysis, we   give the definitions of Lipschitz constant on the original gradient, coordinated smooth function, mixtured gradient of the coordinated smooth functions, Lipschitz constant on the mixtured gradient,  and the explanation of $x^s_t$ used in the analysis as follows, which are  critical to  the  analysis of AsyDSZOVR.

 \begin{enumerate}
\item \noindent \textbf{Lipschitz constant on the original gradient:} \ \  For the smooth functions  $f_i$, we have the Lipschitz constant $L$ for $\nabla f_i$  as following.
\begin{assumption}
\label{definition1}
$L $ is the  Lipschitz constant for $\nabla f_i$ ($\forall i \in \{1,\cdots,l\}$) in (\ref{formulation1}). Thus,  $\forall x$ and $\forall y$, $L$-Lipschitz smooth  can be presented as
   \begin{eqnarray}\label{ass3}
\| \nabla f_i(x) - \nabla f_i(y) \|  \leq L \|x - y \|
\end{eqnarray}
Equivalently, $L$-Lipschitz smooth  can  also be written as the formulation (\ref{coordinate_lipschitz_constant2}).
   \begin{eqnarray} \label{coordinate_lipschitz_constant2}
f_i(x) \leq f_i(y) + \langle \nabla f_i(y) , x-y \rangle + \frac{L}{2}  \left \| x-y \right \|^2
\end{eqnarray}

\end{assumption}

\item \textbf{Coordinated smooth function:} Given a function $f(x)$ and a predefined approximation parameter vector $[\mu_1,\mu_2,\cdots,\mu_N]$,  we define a coordinated smooth function $f^j(x)$ w.r.t the $j$-th dimension which was used in  \citep{lian2016comprehensive}.
\begin{eqnarray}\label{equdef2}
f^j(x)=\mathbb{E}_{v\sim U_{[-\mu_j,\mu_j]}}(p(x+ve_j)) = \frac{1}{2\mu_j}\int_{-\mu_j}^{\mu_j} f(x+ve_j) dv
\end{eqnarray}
where ${v\sim U_{[-\mu_j,\mu_j]}}$ means that $v$ follows the uniform distribution over the interval $[-\mu_j,\mu_j]$.  It should be noted that, we have the following equation between $G_j(x,f)$ and $\nabla_j f^j(x)$.
\begin{eqnarray}\label{equdef3}
\nabla f^j(x) &=& \frac{1}{2\mu_j}\int_{-\mu_j}^{\mu_j} \nabla_j   f(x+ve_j) dv
\\ & = & \nonumber \frac{1}{2  \mu_j } \left ( f_i (x+ \mu_j e_j) - f_i (x- \mu_j e_j) \right )e_j = N G_j(x,f)
\end{eqnarray}
In addition, we have
\begin{eqnarray}\label{equdef3.1}
\mathbb{E}_j \|\nabla_j f^j(x) - \nabla_j f(x) \| \leq  \frac{L^2 \sum_{j=1}^N \mu_j^2 }{4N} \stackrel{\rm def}{=} \frac{\omega}{4}
\end{eqnarray}
which is proved in (26) of \citep{lian2016comprehensive}.

\item \noindent \textbf{Mixtured gradient of the coordinated smooth functions:} \ \ Based on the coordinated smooth function $f^j(x)$, we define a mixtured gradient on the coordinated smooth functions as $\sum_{j=1}^N \nabla_j f^j(x)$.

\item \noindent \textbf{Lipschitz constant on the mixtured gradient:} \ \
We assume that    there exists a  Lipschitz constant ($\widetilde{L}$) on  the mixtured gradient as follows.
\begin{assumption}
\label{definition2}
$\widetilde{L}  $ is the  Lipschitz constant for the mixtured gradient $\sum_{j=1}^N \nabla_j f^j(x)$, such that, $\forall x$ and $\forall y$, we have
   \begin{eqnarray}\label{ass4}
\left \| \sum_{j=1}^N \nabla_j f^j(x) - \sum_{j=1}^N \nabla_j f^j(y) \right  \|  \leq \widetilde{L} \|x - y \|
\end{eqnarray}
\end{assumption}
Because $f^j(x)$ is a smooth function of $f(x)$, it is reasonable to have a Lipschitz constant on the mixtured gradient. Specifically, if  $[\mu_1,\mu_2,\cdots,\mu_N]=\textbf{0}$, it is easy to verify that $\widetilde{L}=L$. If $\mu_j= \infty$ for all $j=1,\cdots,N$, it is easy to verify that $\widetilde{L}=0$.  Note that, it is possible that $\widetilde{L} > L$.

Correspondingly,   we assume there exists a relationship constant $\widehat{L}$ between the original gradient and the mixtured gradient, as follows. Note that, it is also possible that $\widehat{L} > 1$.
\begin{assumption}
\label{definition2} For a smooth function $f$, we have the relationship constant $\widehat{L}$ between the original gradient and the mixtured gradient as
   \begin{eqnarray}\label{ass5}
\left \| \sum_{j=1}^N \nabla_j f^j(x)   \right  \|  \leq \widehat{L} \| \nabla f(x)  \|
\end{eqnarray}
\end{assumption}

\item \textbf{$x^s_t$:} As mentioned previously, AsySBCDVR does not use any locks in the reading and writing. Thus,  in the line 10 of Algorithm \ref{algorithm2}, $x^s_t$ (left side of `$\leftarrow$') updated in the shared memory may be inconsistent with the ideal one (right side of `$\leftarrow$') computed by the proximal operator. In the analysis, we use $x^s_t$ to denote the ideal one computed by the proximal operator.  Same as mentioned in \citep{mania2015perturbed}, there might not be an actual time the ideal ones exist in the
shared memory, except the first and last  iterates for each  outer loop. It is noted that, $x^s_0$ and $x^s_m$ are exactly what is stored in shared memory. Thus, we only consider the ideal $x^s_t$ in the analysis.
\end{enumerate}
Then, we give the upper bounds of $\mathbb{E}  \left \|   G(x;f_i) -  G(y;f_i) \right \|^2$ and $\sum_{t=0}^{m-1} \mathbb{E}  \left \|  \widehat{v}^{s+1}_{t} \right \|^2 $ in  Lemma \ref{lemma0.9} and \ref{lemma1} respectively. Based on Lemma \ref{lemma0.9} and \ref{lemma1}, we  give an upper bound of $\sum_{t=0}^{m-1} \mathbb{E}  \left \| \nabla f({x}^{s+1}_t)   \right \|^2$ (Theorem \ref{theorem1}). Then, we prove  the sublinear rate of the convergence  (Theorem \ref{theorem2} and Corollary \ref{corollary1}).

\begin{lemma} \label{lemma0.9}
For the smooth function $f_i$ and the corresponding approximate full gradient $G(x;f_i) $, we have
\begin{eqnarray}\label{equ_lemma0.9_o.1}
\mathbb{E}  \left \|   G(x;f_i) -  G(y;f_i) \right \|^2 \leq  \widetilde{L}^2 \left \|x - y \right \|^2
\end{eqnarray}

\end{lemma}
\begin{proof}
Based on the definition of the approximate gradient $G(x;f_i) $, we have that
\begin{eqnarray}\label{equ_lemma0.9_1}
&& \mathbb{E}  \left \|   G(x;f_i) -  G(y;f_i) \right \|^2
 =    \mathbb{E}  \left \| \frac{1}{N} \sum_{j =1}^N \left ( G_j(x;f_i) -  G_j(y;f_i) \right ) \right \|^2
\\ & = & \nonumber    \mathbb{E}  \left \| \sum_{j=1}^N \left ( \nabla_j f_i^{j}(x) -  \nabla_j f_i^{j}(y) \right ) \right \|^2
{\leq}     \widetilde{L}^2 \left \|x - y \right \|^2 {add \ what \ is \ \widetilde{L}}
\end{eqnarray}
where the second equality uses (\ref{equdef3}), the first inequality uses (\ref{ass4}).  This completes the proof.
\end{proof}
\begin{lemma} \label{lemma1}
If $Y-{2N\widetilde{L}^2 \gamma^2 \tau^2} >0$, we have that
\begin{eqnarray}\label{equ_lemma1_o.1}
 \sum_{t=0}^{m-1} \mathbb{E}  \left \|  \widehat{v}^{s+1}_{t} \right \|^2 \leq \frac{2Y}{Y-{2N\widetilde{L}^2 \gamma^2 \tau^2}} \sum_{t=0}^{m-1} \left (    \frac{2 N \widetilde{L}^2 }{b}\left \|{x}^{s+1}_t - \widetilde{x}^s \right \|^2  +  2 \widehat{L} \mathbb{E}  \left \| \nabla f({x}^{s+1}_t) \right \|^2  \right )
\end{eqnarray}
\end{lemma}
\begin{proof}
Let $ {v}^{s+1}_t = \frac{1}{b}\sum_{i\in \mathcal{B}(t)} G({x}^{s+1}_t;f_i)- \frac{1}{b}\sum_{i\in \mathcal{B}(t)} G(\widetilde{x}^s;f_i)  + G(\widetilde{x}^s;f)$, we have that
\begin{eqnarray}\label{equ_lemma1_1}
 && \mathbb{E}  \left \|  \widehat{v}^{s+1}_t \right \|^2 =  \mathbb{E}  \left \|  \widehat{v}^{s+1}_t - {v}^{s+1} + {v}^{s+1}_t  \right \|^2
 \\ & \leq & \nonumber  2 \mathbb{E}  \left \|  \widehat{v}^{s+1}_t - {v}^{s+1}_t \right \|^2 + 2\mathbb{E}  \left \|  {v}^{s+1}_t  \right \|^2
 \\ & = & \nonumber  2 \mathbb{E}  \left \|  \frac{1}{b}\sum_{i \in \mathcal{B}(t)} \left ( G(\widehat{x}^{s+1}_t;f_{i}) -  G({x}^{s+1}_t;f_i) \right ) \right \|^2 + 2 \mathbb{E}  \left \|  {v}^{s+1}_t  \right \|^2
  \\ & \leq & \nonumber  \frac{2}{b} \sum_{i\in \mathcal{B}(t)}  \mathbb{E}  \left \|   G(\widehat{x}^{s+1}_t;f_i) -  G({x}^{s+1}_t;f_i) \right \|^2 + 2 \mathbb{E}  \left \|  {v}^{s+1}_t  \right \|^2
\\ & {\leq} & \nonumber  {2 \widetilde{L}^2}   \mathbb{E}  \left \| \widehat{x}^{s+1}_t -  {x}^{s+1}_t  \right \|^2  + 2 \mathbb{E}  \left \|  {v}^{s+1}_t  \right \|^2
\\ & = & \nonumber {2 \widetilde{L}^2 \gamma^2} \mathbb{E}  \left \| \sum_{t' \in K(t)}  B_{t'}^{s+1} \widehat{v}^{s+1}_{J(t')} \right \|^2  + 2 \mathbb{E}  \left \|  {v}^{s+1}_t  \right \|^2
\\ & \leq & \nonumber {2\widetilde{L}^2 \gamma^2} \tau \mathbb{E} \sum_{t' \in K(t)}   \left \| B_{t'}^{s+1} \widehat{v}^{s+1}_{J(t')} \right \|^2  + 2 \mathbb{E}  \left \|  {v}^{s+1}_t  \right \|^2
\\ & \leq & \nonumber {2\widetilde{L}^2 \gamma^2} \tau \sum_{t' \in K(t)}  \mathbb{E}  \left \|  \widehat{v}^{s+1}_{J(t')} \right \|^2  + 2 \mathbb{E}  \left \|  {v}^{s+1}_t  \right \|^2
\\ & = & \nonumber \frac{2N\widetilde{L}^2 \gamma^2 \tau}{Y}   \sum_{t' \in K(t)}  \mathbb{E} \left \|  \widehat{v}^{s+1}_{t'} \right \|^2  + 2 \mathbb{E}  \left \|  {v}^{s+1}_t  \right \|^2
\end{eqnarray}
 where the first, second and fourth inequalities use the fact that $\| \sum_{i=1}^n a_i \|^2 \leq n \sum_{i=1}^n \| a_i \|^2 $, the third inequality uses (\ref{equ_lemma0.9_o.1}), the fifth inequality uses the Cauchy-Schwarz inequality and the fact $\left \| B_{t}^{s+1} \right \| \leq 1$. We consider a fixed stage $s+1$ such that $x_0^{s+1} = x_{m}^{s}$. By summing
the the inequality (\ref{equ_lemma1_1}) over $t = 0,\cdots,m-1$,  we obtain
 \begin{eqnarray} \label{equ_lemma1_2}
  \sum_{t=0}^{m-1} \mathbb{E}  \left \|  \widehat{v}^{s+1}_{t} \right \|^2 & \leq &   \sum_{t=0}^{m-1} \left ( \frac{2N \widetilde{L}^2 \gamma^2 \tau}{Y}   \sum_{t' \in K(t)} \mathbb{E}  \left \| \widehat{v}^{s+1}_{t'} \right \|^2  + 2 \mathbb{E}  \left \|  {v}^{s+1}_t  \right \|^2 \right )
\\ & \leq & \nonumber   \frac{2N\widetilde{L}^2 \gamma^2 \tau^2 }{Y}   \sum_{t=0}^{m-1}  \mathbb{E}  \left \| \widehat{v}^{s+1}_t \right \|^2  + 2 \sum_{t=0}^{m-1} \mathbb{E}  \left \|  {v}^{s+1}_t  \right \|^2
  \end{eqnarray}
where the second inequality uses the Assumption \ref{ass0.1}. If $Y-{2N\widetilde{L}^2 \gamma^2 \tau^2} >0$, we have that
 \begin{eqnarray} \label{equ_lemma1_3}
  \sum_{t=0}^{m-1} \mathbb{E}  \left \|  \widehat{v}^{s+1}_{t} \right \|^2 & \leq &  \frac{2Y}{Y-{2N\widetilde{L}^2 \gamma^2 \tau^2}} \sum_{t=0}^{m-1} \mathbb{E}  \left \|  {v}^{s+1}_t  \right \|^2
  \end{eqnarray}
 We next bound $\mathbb{E}  \left \|  {v}^{s+1}_t  \right \|^2$ by
 \begin{eqnarray} \label{equ_lemma1_4}
&& \mathbb{E}  \left \|  {v}^{s+1}_t  \right \|^2
\\ & = & \nonumber  \mathbb{E}  \left \|  \frac{1}{b}\sum_{i\in \mathcal{B}(t)} G({x}^{s+1}_t;f_i)- \frac{1}{b}\sum_{i\in \mathcal{B}(t)} G(\widetilde{x}^s;f_i)  + G(\widetilde{x}^s;f)  \right \|^2
\\ & = & \nonumber  \mathbb{E}  \left \|  \frac{1}{b}\sum_{i\in \mathcal{B}(t)} G({x}^{s+1}_t;f_i)- \frac{1}{b}\sum_{i\in \mathcal{B}(t)} G(\widetilde{x}^s;f_i)  + G({x}^s;f) - G({x}^{s+1}_t;f) + G(\widetilde{x}^{s+1}_t;f) \right \|^2
\\ & \leq & \nonumber 2 \mathbb{E}  \left \|  \frac{1}{b}\sum_{i\in \mathcal{B}(t)} G({x}^{s+1}_t;f_i)- \frac{1}{b}\sum_{i\in \mathcal{B}(t)} G(\widetilde{x}^s;f_i)  - \left ( G({x}^{s+1}_t;f) -G(\widetilde{x}^s;f) \right ) \right \|^2 + 2 \mathbb{E}  \left \|   G({x}^{s+1}_t;f) \right \|^2
\\ & = & \nonumber \frac{2}{b^2} \mathbb{E}  \left \|  \sum_{i\in \mathcal{B}(t)} \left ( G({x}^{s+1}_t;f_i)-  G(\widetilde{x}^s;f_i)  - \left ( G({x}^{s+1}_t;f) -G(\widetilde{x}^s;f) \right ) \right )  \right \|^2 + 2 \mathbb{E}  \left \|  \frac{1}{N} \sum_{j=1}^N  G_j({x}^{s+1}_t;f) \right \|^2
\\ & \leq & \nonumber \frac{2}{b} \mathbb{E}  \left \|  G({x}^{s+1}_t;f_i)-  G(\widetilde{x}^s;f_i)  -  G({x}^{s+1}_t;f) -G(\widetilde{x}^s;f)  \right \|^2 + 2 \mathbb{E}  \left \|  \frac{1}{N} \sum_{j=1}^N  G_j({x}^{s+1}_t;f) \right \|^2
\\ & \leq & \nonumber \frac{2}{b} \mathbb{E}  \left \|  G({x}^{s+1}_t;f_i)-  G(\widetilde{x}^s;f_i)  \right \|^2 +  2 \mathbb{E}  \left \| \sum_{j=1}^N  \nabla_jf^j({x}^{s+1}_t) \right \|^2
\\ & {\leq} & \nonumber  \frac{2  \widetilde{L}^2 }{b}\left \|{x}^{s+1}_t - \widetilde{x}^s \right \|^2  + 2 \widehat{L} \mathbb{E}  \left \| \nabla f({x}^{s+1}_t) \right \|^2
  \end{eqnarray}
where the first inequality uses $\| \sum_{i=1}^n a_i \|^2 \leq n \sum_{i=1}^n \| a_i \|^2 $,  The second inequality uses Lemma 7 in \citep{reddi2016fast}, the third inequality uses $\mathbb{E} \|x- \mathbb{E}x\|^2 \leq \mathbb{E} \|x\|^2$, the fourth inequality uses (\ref{equ_lemma0.9_o.1}) and (\ref{ass5}).
 This completes the proof.
\end{proof}

\begin{theorem} \label{theorem1}
Setting $c_m=0$,  $\beta_t>0$. Let \begin{eqnarray}
c_t&=&c_{t+1}    (1+ \gamma \beta_t) + \left ( \frac{c_{t+1}N\gamma^2}{Y} +   \frac{L Y \gamma^2 }{2N} + \frac{\gamma^3 NL^2 \tau^2}{Y}  \right )  \frac{4Y N \widetilde{L}^2}{b(Y-{2N\widetilde{L}^2 \gamma^2 \tau^2})}
\\ \Gamma_t &=& \frac{\gamma}{2} - \left ( \frac{c_{t+1}N\gamma^2}{Y} +   \frac{L Y \gamma^2 }{2N} + \frac{\gamma^3 NL^2 \tau^2}{Y}  \right )  \frac{4 Y  \widehat{L}}{Y-{2N\widetilde{L}^2 \gamma^2 \tau^2}} \end{eqnarray}
 Let $\eta_t$, $\beta_t$ and $c_{t+1}$ be chosen such that $\Gamma_t>0$ and ${  \beta_t\geq 2 c_{t+1} }$. $\sum_{t=0}^{m-1} \mathbb{E}  \left \| \nabla f({x}^{s+1}_t)   \right \|^2$ in AsyDSZOVR satisfy the bound
\begin{eqnarray}\label{equ_lemma3_o.1}
\sum_{t=0}^{m-1} \mathbb{E}  \left \| \nabla f({x}^{s+1}_t)   \right \|^2 \leq \frac{  \mathbb{E} ( f(x^{s}) ) -  \mathbb{E} ( f(x^{s+1})) +  \frac{\gamma N \omega m}{4}  }{\min_{t\in \{0,\cdots,m-1\} }\Gamma_t}
\end{eqnarray}
\end{theorem}
\begin{proof}
We first bound $\mathbb{E} \left \| x_{t+1}^{s+1} - \widetilde{x}^{s}  \right \|^2$.
\begin{eqnarray}\label{equ_lemma3_o.1}
 && \mathbb{E} \left \| x_{t+1}^{s+1} - \widetilde{x}^{s}  \right \|^2 = \mathbb{E} \left \| x_{t+1}^{s+1} -  x_{t}^{s+1} + x_{t}^{s+1} -  \widetilde{x}^{s}  \right \|^2
 \\ & = & \nonumber \mathbb{E} \left ( \left \| x_{t+1}^{s+1} -  x_{t}^{s+1} \right \|^2  + \left \|  x_{t}^{s+1} - \widetilde{x}^{s}  \right \|^2 + 2 \left \langle  x_{t+1}^{s+1} -  x_{t}^{s+1}, x_{t}^{s+1}  - \widetilde{x}^{s} \right \rangle \right )
 \\ & = & \nonumber \mathbb{E} \left ( \gamma^2 \left \|  \widehat{v}^{s+1}_{J(t)}  \right \|^2  + \left \|  x_{t}^{s+1} - \widetilde{x}^{s}  \right \|^2 - 2 \gamma \left \langle  \widehat{v}^{s+1}_{J(t)} , x_{t}^{s+1}  - \widetilde{x}^{s} \right \rangle \right )
  \\ & = & \nonumber \frac{N\gamma^2}{Y} \mathbb{E}  \left \|  \widehat{v}^{s+1}_{t}  \right \|^2  + \mathbb{E} \left \|  x_{t}^{s+1} - \widetilde{x}^{s}  \right \|^2 - 2 \gamma \mathbb{E}  \left \langle    \frac{1}{b}\sum_{i \in \mathcal{B}(t)}  G(\widehat{x}^{s+1}_t;f_{i}) , x_{t}^{s+1}  - \widetilde{x}^{s} \right \rangle
\\ & \leq & \nonumber \frac{N\gamma^2}{Y} \mathbb{E}  \left \|  \widehat{v}^{s+1}_{t}  \right \|^2   + \mathbb{E} \left \|  x_{t}^{s+1} - \widetilde{x}^{s}  \right \|^2 + 2 \gamma \mathbb{E}  \left (  \frac{1}{2\beta_t } \left \| \frac{1}{b}\sum_{i \in \mathcal{B}(t)}G(\widehat{x}^{s+1}_t;f_{i}) \right \|^2 + \frac{\beta_t}{2}\left \|  x_{t}^{s+1}  - \widetilde{x}^{s}  \right \|^2  \right )
\\ & = & \nonumber \frac{N\gamma^2}{Y} \mathbb{E}  \left \|  \widehat{v}^{s+1}_{t}  \right \|^2   + (1+ \gamma \beta_t)\mathbb{E} \left \|  x_{t}^{s+1} - \widetilde{x}^{s}  \right \|^2 + 2 \gamma \mathbb{E}  \left (  \frac{1}{2\beta_t } \left \| \frac{1}{b}\sum_{i \in \mathcal{B}(t)} \sum_{j=1}^N  \nabla_jf^j({x}^{s+1}_t) \right \|^2  \right )
\\ & \leq & \nonumber \frac{N\gamma^2}{Y} \mathbb{E}  \left \|  \widehat{v}^{s+1}_{t}  \right \|^2   + (1+ \gamma \beta_t)\mathbb{E} \left \|  x_{t}^{s+1} - \widetilde{x}^{s}  \right \|^2 +  \frac{ \gamma}{b\beta_t} \mathbb{E}   \left ( \sum_{i \in \mathcal{B}(t)} \left \|  \sum_{j=1}^N  \nabla_jf^j({x}^{s+1}_t) \right \|^2  \right )
\\ & = & \nonumber \frac{N\gamma^2}{Y} \mathbb{E}  \left \|  \widehat{v}^{s+1}_{t}  \right \|^2   + (1+ \gamma \beta_t)\mathbb{E} \left \|  x_{t}^{s+1} - \widetilde{x}^{s}  \right \|^2 + \frac{\gamma N}{\beta_t }   \mathbb{E}  \left \| \nabla_jf^j({x}^{s+1}_t) \right \|^2
\end{eqnarray}
where the first inequality uses the Young's inequality,  the second inequality uses  the fact that $\| \sum_{i=1}^n a_i \|^2 \leq n \sum_{i=1}^n \| a_i \|^2 $. We next bound $\mathbb{E} \left \| \nabla_j f(x^{s+1}_{t}) -   \nabla_jf^j(\widehat{x}^{s+1}_t)　\right \|^2$.
\begin{eqnarray}\label{equ_lemma3_o.2}
&&\mathbb{E} \left \| \nabla_j f(x^{s+1}_{t}) -   \nabla_jf^j(\widehat{x}^{s+1}_t)　\right \|^2
\\ & = & \nonumber \mathbb{E} \left \| \nabla_j f(x^{s+1}_{t}) -   \nabla_jf(\widehat{x}^{s+1}_t) +  \nabla_jf(\widehat{x}^{s+1}_t) - \nabla_jf^j(\widehat{x}^{s+1}_t)　\right \|^2
\\ & \leq & \nonumber 2 \mathbb{E} \left \| \nabla_j f(x^{s+1}_{t}) -   \nabla_jf(\widehat{x}^{s+1}_t) \right \|^2 + 2 \mathbb{E} \left \| \nabla_jf(\widehat{x}^{s+1}_t) - \nabla_jf^j(\widehat{x}^{s+1}_t)　\right \|^2
\\ & \leq & \nonumber \frac{2}{N} \mathbb{E} \left \| \nabla f(x^{s+1}_{t}) -   \nabla f(\widehat{x}^{s+1}_t) \right \|^2 + \frac{\omega}{2}
\\ & \leq & \nonumber \frac{2L^2}{N}  \left \| x^{s+1}_{t} - \widehat{x}^{s+1}_t  \right \|^2 + \frac{\omega}{2}
\\ & = & \nonumber \frac{2L^2 \gamma^2 }{N}  \left \|  \sum_{t' \in K(t)}  B_{t'}^{s+1} \widehat{v}^{s+1}_{J(t')}  \right \|^2 + \frac{\omega}{2}
\\ & \leq & \nonumber \frac{2L^2  \gamma^2  \tau}{N} \mathbb{E} \sum_{t' \in K(t)}   \left \| B_{t'}^{s+1} \widehat{v}^{s+1}_{J(t')} \right \|^2  + \frac{\omega}{2}
\\ & \leq & \nonumber \frac{2L^2  \gamma^2  \tau}{N} \mathbb{E} \sum_{t' \in K(t)}   \left \|  \widehat{v}^{s+1}_{J(t')} \right \|^2  + \frac{\omega}{2}
\\ & = & \nonumber \frac{2L^2  \gamma^2  \tau}{Y}  \sum_{t' \in K(t)}  \mathbb{E} \left \|  \widehat{v}^{s+1}_{t'} \right \|^2  + \frac{\omega}{2}
\end{eqnarray}
where the first and fourth inequalities use $\| \sum_{i=1}^n a_i \|^2 \leq n \sum_{i=1}^n \| a_i \|^2 $, the second inequality uses (\ref{equdef3.1}), the third inequality uses (\ref{ass3}), the fifth inequality uses the  Cauchy-Schwarz inequality and the fact $\left \| B_{t}^{s+1} \right \| \leq 1$. We bound $\mathbb{E}\left (f(x^{s+1}_{t+1}) \right )$ as follows.
\begin{eqnarray}\label{equ_lemma3_o.3}
&& \mathbb{E}\left (f(x^{s+1}_{t+1}) \right )
\\ & \leq & \nonumber  \mathbb{E}\left (f(x^{s+1}_{t}) + \left \langle \nabla f(x^{s+1}_{t}), x^{s+1}_{t+1} - x^{s+1}_{t} \right \rangle + \frac{L}{2}\left \| x^{s+1}_{t+1} - x^{s+1}_{t} \right \|^2    \right )
\\ & = & \nonumber  \mathbb{E}\left (f(x^{s+1}_{t}) - \gamma \left \langle \nabla f(x^{s+1}_{t}),  \widehat{v}^{s+1}_{J(t)} \right \rangle  + \frac{L \gamma^2 }{2}\left \| \widehat{v}^{s+1}_{J(t)} \right \|^2   \right )
\\ & = & \nonumber  \mathbb{E} f(x^{s+1}_{t}) - \gamma \mathbb{E} \left \langle \nabla f(x^{s+1}_{t}),  \frac{1}{b}\sum_{i \in \mathcal{B}(t)}  G(\widehat{x}^{s+1}_t;f_{i}) \right \rangle  + \frac{L Y \gamma^2 }{2N} \mathbb{E} \left \| \widehat{v}^{s+1}_{t} \right \|^2
\\ & = & \nonumber  \mathbb{E} f(x^{s+1}_{t}) - \gamma \mathbb{E} \left \langle \nabla f(x^{s+1}_{t}),  \frac{1}{N} \sum_{j=1}^N  G_j(\widehat{x}^{s+1}_t;f) \right \rangle  + \frac{L Y \gamma^2 }{2N} \mathbb{E} \left \| \widehat{v}^{s+1}_{t} \right \|^2
\\ & = & \nonumber  \mathbb{E} f(x^{s+1}_{t}) - \gamma \mathbb{E} \left \langle \nabla f(x^{s+1}_{t}),  \sum_{j=1}^N  \nabla_jf^j(\widehat{x}^{s+1}_t) \right \rangle  + \frac{L Y \gamma^2 }{2N} \mathbb{E} \left \| \widehat{v}^{s+1}_{t} \right \|^2
\\ & = & \nonumber  \mathbb{E} f(x^{s+1}_{t})  + \frac{L Y \gamma^2 }{2N} \mathbb{E} \left \| \widehat{v}^{s+1}_{t} \right \|^2
\\ &  & \nonumber - \frac{\gamma}{2} \left ( \mathbb{E} \left \|  \nabla f(x^{s+1}_{t}) \right \|^2 +  \mathbb{E} \left \| \sum_{j=1}^N  \nabla_jf^j(\widehat{x}^{s+1}_t) \right  \|^2 - \mathbb{E}　\left \| \nabla f(x^{s+1}_{t}) -  \sum_{j=1}^N  \nabla_jf^j(\widehat{x}^{s+1}_t)  \right \|^2  \right )
\\ & = & \nonumber  \mathbb{E} f(x^{s+1}_{t}) +  \frac{L Y \gamma^2 }{2N} \mathbb{E} \left \| \widehat{v}^{s+1}_{t} \right \|^2 - \frac{\gamma}{2} \mathbb{E} \left \|  \nabla f(x^{s+1}_{t}) \right \|^2   -  \frac{\gamma N}{2} \mathbb{E} \left \|  \nabla_jf^j(\widehat{x}^{s+1}_t) \right  \|^2
\\ &  & \nonumber +　\frac{\gamma N}{2} \mathbb{E} \left \| \nabla_j f(x^{s+1}_{t}) -   \nabla_jf^j(\widehat{x}^{s+1}_t)　\right \|^2  　
\\ & \leq & \nonumber  \mathbb{E} f(x^{s+1}_{t}) +  \frac{L Y \gamma^2 }{2N} \mathbb{E} \left \| \widehat{v}^{s+1}_{t} \right \|^2 - \frac{\gamma}{2} \mathbb{E} \left \|  \nabla f(x^{s+1}_{t}) \right \|^2   -  \frac{\gamma N}{2} \mathbb{E} \left \|  \nabla_jf^j(\widehat{x}^{s+1}_t) \right  \|^2
\\ &  & \nonumber +　\frac{\gamma^3 NL^2 \tau}{Y}   \sum_{t' \in K(t)}  \mathbb{E} \left \|  \widehat{v}^{s+1}_{t'} \right \|^2  + \frac{\gamma N \omega}{4}
\end{eqnarray}
where the first inequality uses (\ref{coordinate_lipschitz_constant2}), the second inequality uses (\ref{equ_lemma3_o.2}). Next, we define Lyapunov function $R^{s+1}_{t} = \mathbb{E}  \left ( f({x}^{s+1}_t)+ c_t  \left \| x_{t}^{s+1} - \widetilde{x}^{s}  \right \|^2 \right ) $, and give the upper bound of $R^{s+1}_{t+1}$ as follows.
\begin{eqnarray}\label{equ_lemma3_o.4}
&& R^{s+1}_{t+1}
\\ & = & \nonumber  \mathbb{E}  \left ( f({x}^{s+1}_{t+1})+ c_{t+1}  \left \| x_{t+1}^{s+1} - \widetilde{x}^{s}  \right \|^2 \right )
\\ & \leq & \nonumber   \mathbb{E} f(x^{s+1}_{t}) +  \frac{L Y \gamma^2 }{2N} \mathbb{E} \left \| \widehat{v}^{s+1}_{t} \right \|^2 - \frac{\gamma}{2} \mathbb{E} \left \|  \nabla f(x^{s+1}_{t}) \right \|^2   -  \frac{\gamma N}{2} \mathbb{E} \left \|  \nabla_jf^j(\widehat{x}^{s+1}_t) \right  \|^2
\\ &  & \nonumber +　\frac{\gamma^3 NL^2 \tau}{Y}   \sum_{t' \in K(t)}  \mathbb{E} \left \|  \widehat{v}^{s+1}_{t'} \right \|^2  + \frac{\gamma N \omega}{4}
\\ &  & \nonumber + c_{t+1}  \left ( \frac{N\gamma^2}{Y} \mathbb{E}  \left \|  \widehat{v}^{s+1}_{t}  \right \|^2   + (1+ \gamma \beta_t)\mathbb{E} \left \|  x_{t}^{s+1} - \widetilde{x}^{s}  \right \|^2 + \frac{\gamma N}{\beta_t }   \mathbb{E}  \left \| \nabla_jf^j({x}^{s+1}_t) \right \|^2 \right )
\\ & = & \nonumber   \mathbb{E} f(x^{s+1}_{t}) +  \frac{L Y \gamma^2 }{2N} \mathbb{E} \left \| \widehat{v}^{s+1}_{t} \right \|^2 - \frac{\gamma}{2} \mathbb{E} \left \|  \nabla f(x^{s+1}_{t}) \right \|^2   -  \left (　\frac{\gamma N}{2} - \frac{  c_{t+1} \gamma N}{\beta_t }　\right ) \mathbb{E} \left \|  \nabla_jf^j(\widehat{x}^{s+1}_t) \right  \|^2
\\ &  & \nonumber +　\frac{\gamma^3 NL^2 \tau}{Y}   \sum_{t' \in K(t)}  \mathbb{E} \left \|  \widehat{v}^{s+1}_{t'} \right \|^2 + \frac{c_{t+1}N\gamma^2}{Y} \mathbb{E}  \left \|  \widehat{v}^{s+1}_{t}  \right \|^2  + c_{t+1}    (1+ \gamma \beta_t) \mathbb{E} \left \|  x_{t}^{s+1} - \widetilde{x}^{s}  \right \|^2 + \frac{\gamma N \omega}{4}
\\ & {\leq } & \nonumber   \mathbb{E} f(x^{s+1}_{t}) +  \frac{L Y \gamma^2 }{2N} \mathbb{E} \left \| \widehat{v}^{s+1}_{t} \right \|^2 - \frac{\gamma}{2} \mathbb{E} \left \|  \nabla f(x^{s+1}_{t}) \right \|^2  +　\frac{\gamma^3 NL^2 \tau}{Y}   \sum_{t' \in K(t)}  \mathbb{E} \left \|  \widehat{v}^{s+1}_{t'} \right \|^2
\\ &  & \nonumber + \frac{c_{t+1}N\gamma^2}{Y} \mathbb{E}  \left \|  \widehat{v}^{s+1}_{t}  \right \|^2  + c_{t+1}    (1+ \gamma \beta_t) \mathbb{E} \left \|  x_{t}^{s+1} - \widetilde{x}^{s}  \right \|^2 + \frac{\gamma N \omega}{4}
\end{eqnarray}
where the  first inequality uses (\ref{equ_lemma3_o.1}) and (\ref{equ_lemma3_o.3}), and the second inequality uses the constraint ${  \beta_t\geq 2 c_{t+1} }$. We consider a fixed stage $s+1$ such that $x_0^{s+1} = x_{m}^{s}$. By summing
the the inequality (\ref{equ_lemma3_o.4}) over $t = 0,\cdots,m-1$,  we obtain
\begin{eqnarray} \label{equ_lemma3_o.5}
&&  \sum_{t=0}^{m-1} R^{s+1}_{t+1}
\\ & {\leq } & \nonumber \sum_{t=0}^{m-1} \left (  \mathbb{E} f(x^{s+1}_{t}) +  \frac{L Y \gamma^2 }{2N} \mathbb{E} \left \| \widehat{v}^{s+1}_{t} \right \|^2 - \frac{\gamma}{2} \mathbb{E} \left \|  \nabla f(x^{s+1}_{t}) \right \|^2  +　\frac{\gamma^3 NL^2 \tau}{Y}   \sum_{t' \in K(t)}  \mathbb{E} \left \|  \widehat{v}^{s+1}_{t'} \right \|^2 \right .
\\ &  & \nonumber \left . + \frac{c_{t+1}N\gamma^2}{Y} \mathbb{E}  \left \|  \widehat{v}^{s+1}_{t}  \right \|^2  + c_{t+1}    (1+ \gamma \beta_t) \mathbb{E} \left \|  x_{t}^{s+1} - \widetilde{x}^{s}  \right \|^2 + \frac{\gamma N \omega}{4} \right )
\\ & = & \nonumber \sum_{t=0}^{m-1} \left (  \mathbb{E} f(x^{s+1}_{t})  - \frac{\gamma}{2} \mathbb{E} \left \|  \nabla f(x^{s+1}_{t}) \right \|^2 + c_{t+1}    (1+ \gamma \beta_t) \mathbb{E} \left \|  x_{t}^{s+1} - \widetilde{x}^{s}  \right \|^2 + \frac{\gamma N \omega}{4}  \right .
\\ &  & \nonumber  \left . + \left ( \frac{c_{t+1}N\gamma^2}{Y} +   \frac{L Y \gamma^2 }{2N} + \frac{\gamma^3 NL^2 \tau^2}{Y}  \right ) \mathbb{E}  \left \|  \widehat{v}^{s+1}_{t}  \right \|^2  \right )
\\ & {\leq } & \nonumber \sum_{t=0}^{m-1} \left (  \mathbb{E} f(x^{s+1}_{t})    + \frac{\gamma N \omega}{4}   \right .
\\ &  & \nonumber  \left . -  \left (\frac{\gamma}{2} - \left ( \frac{c_{t+1}N\gamma^2}{Y} +   \frac{L Y \gamma^2 }{2N} + \frac{\gamma^3 NL^2 \tau^2}{Y}  \right )  \frac{4 Y  \widehat{L}}{Y-{2N\widetilde{L}^2 \gamma^2 \tau^2}}  \right ) \mathbb{E} \left \|  \nabla f(x^{s+1}_{t}) \right \|^2 +
\right .
\\ &  & \nonumber  \left .  \left ( c_{t+1}    (1+ \gamma \beta_t) + \left ( \frac{c_{t+1}N\gamma^2}{Y} +   \frac{L Y \gamma^2 }{2N} + \frac{\gamma^3 NL^2 \tau^2}{Y}  \right )  \frac{4Y N \widetilde{L}^2}{b(Y-{2N\widetilde{L}^2 \gamma^2 \tau^2})}  \right ) \mathbb{E} \left \|  x_{t}^{s+1} - \widetilde{x}^{s}  \right \|^2
\right)
\\ & {= } & \nonumber \sum_{t=0}^{m-1} \left (  R^{s+1}_{t} - \Gamma_t \mathbb{E} \left \|  \nabla f(x^{s+1}_{t}) \right \|^2 +  \frac{\gamma N \omega}{4}   \right )
\end{eqnarray}
where the second inequality uses (\ref{equ_lemma1_o.1}).
Because $c_m=0$, we have that $R^{s+1}_{m}=\mathbb{E} (f(x^{s+1}_{m}) ) =\mathbb{E}( f(x^{s+1}) ) $. In addition, we have that $R^{s+1}_{0}=\mathbb{E} (f(x^{s+1}_{0}) ) =\mathbb{E}( f(x^{s}) ) $.
Based on (\ref{equ_lemma3_o.5}), we have that
\begin{eqnarray}\label{equ_lemma3_o.6}
\sum_{t=0}^{m-1} \mathbb{E}  \left \| \nabla f({x}^{s+1}_t)   \right \|^2 &\leq& \frac{ \sum_{t=0}^{m-1} (R^{s+1}_{t} - R^{s+1}_{t+1}) +  \frac{\gamma N \omega m}{4}  }{\min_{t\in \{0,\cdots,m-1\} }\Gamma_t}
\\ & {= } & \nonumber \frac{ (R^{s+1}_{0} - R^{s+1}_{m}) +  \frac{\gamma N \omega m}{4}  }{\min_{t\in \{0,\cdots,m-1\} }\Gamma_t}
\\ & {= } & \nonumber \frac{  \mathbb{E} ( f(x^{s})) -  \mathbb{E} ( f(x^{s+1})) +  \frac{\gamma N \omega m}{4}  }{\min_{t\in \{0,\cdots,m-1\} }\Gamma_t}
\end{eqnarray}
 This completes the proof.
\end{proof}

\begin{theorem} \label{theorem2}
Let $c_{m} = 0$,
 $\gamma = \frac{u_0b}{\widetilde{L} l^{\alpha}}$, $\beta_t= \frac{\widetilde{L} N^2}{Y }$, $0<\alpha<1$, $0<u_0<1$, and $c_t=c_{t+1}    (1+ \gamma \beta_t) + \left ( \frac{c_{t+1}N\gamma^2}{Y} +   \frac{L Y \gamma^2 }{2N} + \frac{\gamma^3 NL^2 \tau^2}{Y}  \right )  \frac{4Y N \widetilde{L}^2}{b(Y-{2N\widetilde{L}^2 \gamma^2 \tau^2})}$ for $t=0,\cdots,m-1$, $b<l^{\alpha}$. $\frac{1}{T}\sum_{s=0}^{S-1}\sum_{t=0}^{m-1} \mathbb{E}  \left \| \nabla f({x}^{s+1}_t)   \right \|^2$ in AsyDSZOVR satisfy the bound
\begin{eqnarray}\label{equ_theorem2_o.1}
\frac{1}{T}\sum_{s=0}^{S-1}\sum_{t=0}^{m-1} \mathbb{E}  \left \| \nabla f({x}^{s+1}_t)   \right \|^2 \leq \frac{ \widetilde{L} l^{\alpha} \left (  f(x^{0})) -  \mathbb{E} ( f(x^{S})) \right )}{\sigma b T} +  \frac{ N  u_0 \omega}{4\sigma  }
\end{eqnarray}
\end{theorem}
\begin{proof}
Based on the specified values of $\gamma$ and $\beta_t$, we have that
\begin{eqnarray}\label{equ_theorem2_o.2}
 \theta &=& \gamma \beta_t +  \frac{4 N^2 \gamma^2 \widetilde{L}^2}{b(Y-{2N\widetilde{L}^2 \gamma^2 \tau^2})}
= \frac{u_0b}{\frac{Y l^{\alpha}}{N^2}} +  \frac{4u_0^2b}{\frac{Y l^{2\alpha}}{N^2} - \frac{2  \tau^2 u_0^2b^2}{N }  }
\\ & {= } & \nonumber \frac{u_0b N^2}{Y l^{\alpha} } +  \frac{4u_0^2b N^2}{Y l^{2\alpha} - 2 N \tau^2 u_0^2b^2   }
\\ & \leq & \nonumber  \frac{5u_0b N^2}{Y l^{\alpha}}
\end{eqnarray}
where the inequality uses the constraint $Y l^{\alpha} \leq { Y l^{2\alpha} - 2 N \tau^2 u_0^2b^2  }$ by appropriately choosing $\alpha$ and $u_0$. We set $m= \lfloor  \frac{Y l^{\alpha}}{5u_0b N^2} \rfloor$, from the recurrence definition of $c_t$, we have that
\begin{eqnarray}\label{equ_theorem2_o.3}
c_0 &=&  \frac{4Y N \widetilde{L}^2}{b(Y-{2N\widetilde{L}^2 \gamma^2 \tau^2})} \left (  \frac{L Y \gamma^2 }{2N} + \frac{\gamma^3 NL^2 \tau^2}{Y} \right ) \frac{(1+\theta)^m-1}{\theta}
\\ & {= } & \nonumber  \frac{4Y N \widetilde{L}^2}{b(Y-{2N\widetilde{L}^2 \gamma^2 \tau^2})}   \frac{\frac{L Y u_0^2 b^2 }{2N\widetilde{L}^2 l^{2\alpha} } + \frac{  NL^2 \tau^2 u_0^3 b^3}{Y \widetilde{L}^3 l^{3\alpha} }}{ \frac{u_0b N^2}{Y l^{\alpha} } +  \frac{4u_0^2b N^2}{Y l^{2\alpha} - 2 N \tau^2 u_0^2b^2   }} \left ( (1+\theta)^m-1 \right )
\\ & \leq & \nonumber  \frac{4Y N \widetilde{L}^2l^{2\alpha} }{ b(Yl^{2\alpha}-{2N\tau^2 u_0^2 b^2})}  \frac{\frac{LYu_0^2 b^2 }{2N} + \frac{NL^2 \tau^2 u_0^3 b^3}{Y} \left ( Yl^{2\alpha} -2N\tau^2 u_0^2 b^2 \right )}{5u_0^2 b N^2\widetilde{L}^2 l^{2\alpha} } \left ( (1+\theta)^m-1 \right )
\\ & = & \nonumber  \frac{ \frac{2L Y^2 }{N} + 4 N L^2 \tau^2 u_0 b }{5 N  }  \left ( (1+\theta)^m-1 \right )
\\ & \leq & \nonumber  \underbrace{\frac{ \frac{2L Y^2 }{N} + 4 N L^2 \tau^2 u_0 b }{5 N  }  \left ( e -1 \right )  }_{:=\varrho_1}
\end{eqnarray}
where the first inequality uses ${\widetilde{L}^3 l^{3\alpha}}  \geq {\widetilde{L}^2 l^{2\alpha}}$, the second inequality uses the fact $(1+\frac{1}{a})^a$ is increasing for $a>0$, and $\lim_{a\rightarrow \infty} (1+\frac{1}{a})^a =e$,   which is also used in \citep{reddi2015variance}. Let $\widetilde{\Gamma}$ denote the following quantity:
\begin{eqnarray}
\widetilde{\Gamma} = \min_{t\in \{0,\cdots,m-1\} } \frac{\gamma}{2} - \left ( \frac{c_{t+1}N\gamma^2}{Y} +   \frac{L Y \gamma^2 }{2N} + \frac{\gamma^3 NL^2 \tau^2}{Y}  \right )  \frac{4 Y  \widehat{L}}{Y-{2N\widetilde{L}^2 \gamma^2 \tau^2}}
 \end{eqnarray}
 Now we give a  lower bound of $\widetilde{\Gamma}$ as
\begin{eqnarray}
\widetilde{\Gamma} &=& \min_{t\in \{0,\cdots,m-1\} } \frac{\gamma}{2} - \left ( \frac{c_{t+1}N\gamma^2}{Y} +   \frac{L Y \gamma^2 }{2N} + \frac{\gamma^3 NL^2 \tau^2}{Y}  \right )  \frac{4 Y  \widehat{L}}{Y-{2N\widetilde{L}^2 \gamma^2 \tau^2}}
\\ & \geq & \nonumber  \frac{\gamma}{2} - \left ( \frac{c_{0}N\gamma^2}{Y} +   \frac{L Y \gamma^2 }{2N} + \frac{\gamma^3 NL^2 \tau^2}{Y}  \right ) \underbrace{ \frac{4 Y  \widehat{L}}{Y-{2N\widetilde{L}^2 \gamma^2 \tau^2}}}_{:= \varrho_2}
\\ & = & \nonumber  \frac{\gamma}{2} - \left ( \frac{\varrho_1 N\gamma^2}{Y} +   \frac{L Y \gamma^2 }{2N} + \frac{\gamma^3 NL^2 \tau^2}{Y}  \right ) \varrho_2
\\ & \geq & \nonumber \gamma \underbrace{\left ( \frac{1}{2} -  \frac{\varrho_1N \varrho_2 \gamma  }{Y} -  \frac{L Y \varrho_2 \gamma  }{2N} - \frac{\varrho_2 NL^2 \tau^2 \gamma^2 }{Y}    \right )}_{\varrho_3}
\\ & \geq & \nonumber \frac{ \sigma b}{\widetilde{L} l^{\alpha}}
 \end{eqnarray}
where the first inequality holds because $c_{t}$ decrease with $t$, $\varrho_2$ are constants, $\sigma = {\varrho_3} u_0$. For the last inequality, we use the constraint $b<l^{\alpha}$. Thus, we can appropriately  choose a value of $u_0$, such that ${\varrho_3}>0$, and  $\sigma$ is a small value independent to $l$.
\begin{eqnarray}\label{equ_lemma3_o.6}
\frac{1}{T}\sum_{s=0}^{S-1}\sum_{t=0}^{m-1} \mathbb{E}  \left \| \nabla f({x}^{s+1}_t)   \right \|^2
& \leq &  \frac{1}{T}\sum_{s=0}^{S-1} \frac{  \mathbb{E} ( f(x^{s})) -  \mathbb{E} ( f(x^{s+1})) +  \frac{\gamma N \omega m}{4}  }{\widetilde{\Gamma}}
\\ & = & \nonumber \frac{  f(x^{0})) -  \mathbb{E} ( f(x^{S})) +  \frac{\gamma N \omega T}{4}  }{T\widetilde{\Gamma}}
\\ & \leq & \nonumber \frac{ \widetilde{L} l^{\alpha} \left (  f(x^{0})) -  \mathbb{E} ( f(x^{S})) \right )}{\sigma b T} +  \frac{ N  u_0 \omega}{4\sigma  }
\end{eqnarray}
This completes the proof.
\end{proof}

\begin{corollary}\label{corollary1}
 Let $c_{m} = 0$,
 $\gamma = \frac{u_0b}{\widetilde{L} l^{\alpha}}$, $\beta_t= \frac{\widetilde{L} N^2}{Y }$, $0<\alpha<1$, $0<u_0<1$, and $c_t=c_{t+1}    (1+ \gamma \beta_t) + \left ( \frac{c_{t+1}N\gamma^2}{Y} +   \frac{L Y \gamma^2 }{2N} + \frac{\gamma^3 NL^2 \tau^2}{Y}  \right )  \frac{4Y N \widetilde{L}^2}{b(Y-{2N\widetilde{L}^2 \gamma^2 \tau^2})}$ for $t=0,\cdots,m-1$, $b<l^{\alpha}$. If $\omega=0$,  $\frac{1}{T}\sum_{s=0}^{S-1}\sum_{t=0}^{m-1} \mathbb{E}  \left \| \nabla f({x}^{s+1}_t)   \right \|^2$ in AsyDSZOVR satisfy the bound
\begin{eqnarray}\label{equ_theorem2_o.1}
\frac{1}{T}\sum_{s=0}^{S-1}\sum_{t=0}^{m-1} \mathbb{E}  \left \| \nabla f({x}^{s+1}_t)   \right \|^2 \leq \frac{ \widetilde{L} l^{\alpha} \left (  f(x^{0})) -  \mathbb{E} ( f(x^{S})) \right )}{\sigma b T}
\end{eqnarray}
\end{corollary}

\begin{corollary}\label{corollary2}
 Let $c_{m} = 0$,
 $\gamma = \frac{u_0b}{\widetilde{L} l^{\alpha}}$, $\beta_t= \frac{\widetilde{L} N^2}{Y }$, $0<\alpha<1$, $0<u_0<1$, and $c_t=c_{t+1}    (1+ \gamma \beta_t) + \left ( \frac{c_{t+1}N\gamma^2}{Y} +   \frac{L Y \gamma^2 }{2N} + \frac{\gamma^3 NL^2 \tau^2}{Y}  \right )  \frac{4Y N \widetilde{L}^2}{b(Y-{2N\widetilde{L}^2 \gamma^2 \tau^2})}$ for $t=0,\cdots,m-1$, $b<l^{\alpha}$.  $\frac{1}{T}\sum_{s=0}^{S-1}\sum_{t=0}^{m-1} \mathbb{E}  \left \| \nabla f({x}^{s+1}_t)   \right \|^2$ in DSZOVR satisfy the bound
\begin{eqnarray}\label{equ_theorem2_o.3}
\frac{1}{T}\sum_{s=0}^{S-1}\sum_{t=0}^{m-1} \mathbb{E}  \left \| \nabla f({x}^{s+1}_t)   \right \|^2 \leq \frac{ \widetilde{L} l^{\alpha} \left (  f(x^{0})) -  \mathbb{E} ( f(x^{S})) \right )}{\sigma b T} +  \frac{ N  u_0 \omega}{4\sigma  }
\end{eqnarray}
If $\omega=0$,  $\frac{1}{T}\sum_{s=0}^{S-1}\sum_{t=0}^{m-1} \mathbb{E}  \left \| \nabla f({x}^{s+1}_t)   \right \|^2$ in DSZOVR satisfy the bound
\begin{eqnarray}\label{equ_theorem2_o.4}
\frac{1}{T}\sum_{s=0}^{S-1}\sum_{t=0}^{m-1} \mathbb{E}  \left \| \nabla f({x}^{s+1}_t)   \right \|^2 \leq \frac{ \widetilde{L} l^{\alpha} \left (  f(x^{0})) -  \mathbb{E} ( f(x^{S})) \right )}{\sigma b T}
\end{eqnarray}
\end{corollary}
\section{Conclusion}\label{conclusion}
In this paper, we propose an asynchronous doubly stochastic   zeroth-order optimization algorithm using the accelerated technology of variance reduction  (AsyDSZOVR).  Our AsyDSZOVR randomly select a set of samples  and a set of  features simultaneously to handle large scale problems both in volume and dimension.
 Rigorous theoretical analysis show that  the convergence rate can be improved from $O(\frac{1}{\sqrt{T}})$  the  best result of existing algorithms  to  $O(\frac{1}{T})$. Also our theoretical results is an improvement to the ones of the sequential stochastic  zeroth-order optimization algorithms.

\nocite{langley00}


\bibliography{sample}

\end{document}